\documentclass[journal,twoside,web]{ieeecolor}
\usepackage{lcsys}
\usepackage{cite}

\usepackage{amsmath,amssymb,amsfonts}

\usepackage{multirow}
\usepackage{hhline}

\usepackage{url}

\usepackage{graphicx}
\usepackage{diagbox}
\usepackage{textcomp}
\let\labelindent\relax
\usepackage{enumitem}
\usepackage{siunitx}
\usepackage{booktabs}
\usepackage[hidelinks]{hyperref}
\usepackage[noabbrev]{cleveref}

\usepackage{mathtools}
\usepackage{mleftright}
\mleftright

\usepackage{algorithm}
\usepackage[noend]{algpseudocode}
\algrenewcommand\algorithmicindent{1em}

\DeclareMathOperator*{\argmin}{arg\,min}

\newtheorem{lem}{Lemma}
\newtheorem{thm}{Theorem}

\newcommand{\bbN}{\mathbb{N}}
\newcommand{\bbZ}{\mathbb{Z}}
\newcommand{\bbC}{\mathbb{C}}

\newcommand{\calO}{\mathcal{O}}

\newcommand{\calC}{\mathcal{C}}

\newcommand{\bfc}{\mathbf{c}}

\newcommand{\bfz}{\mathbf{Z}}
\newcommand{\bfg}{\mathbf{g}}

\newcommand{\sfB}{\mathsf{B}}

\newcommand{\Mult}{\mathsf{Mult}}

\newcommand{\RotCt}{\mathsf{RotCt}}
\newcommand{\RotVec}{\mathsf{RotVec}}
\newcommand{\Rot}{\mathsf{Rot}}

\newcommand{\Enc}{\mathsf{Enc}}
\newcommand{\Dec}{\mathsf{Dec}}

\newcommand{\Boot}{\mathsf{Boot}}

\usepackage[dvipsnames]{xcolor}

\def\BibTeX{{\rm B\kern-.05em{\sc i\kern-.025em b}\kern-.08em
    T\kern-.1667em\lower.7ex\hbox{E}\kern-.125emX}}
\begin{document}

\title{Relative Entropy Regularized Reinforcement Learning for Efficient Encrypted Policy Synthesis}

\author{Jihoon Suh, Yeongjun Jang, \IEEEmembership{Graduate Student Member, IEEE}, Kaoru Teranishi, \IEEEmembership{Member, IEEE}, and Takashi Tanaka, \IEEEmembership{Senior Member, IEEE}
\thanks{\copyright 2025 IEEE. This is the accepted version of the article: J. Suh, Y. Jang, K. Teranishi, T. Tanaka, ``Relative Entropy Regularized Reinforcement Learning for Efficient Encrypted Policy Synthesis,'' IEEE Control Systems Letters, pp. 1--1, June 2025, doi: 10.1109/LCSYS.2025.3578573. The final published version is available at: https://ieeexplore.ieee.org/document/11030858}
\thanks{J. Suh, K. Teranishi, and T. Tanaka are with the School of Aeronautics and Astronautics, Purdue University, West Lafayette, IN 47907, USA (e-mail: \{suh95, kteranis, tanaka16\}@purdue.edu).}
\thanks{Y. Jang is with ASRI, the Department of Electrical and Computer Engineering, Seoul National University, Seoul, 08826, Korea (e-mail: jangyj@cdsl.kr).}
\thanks{K. Teranishi is also with Japan Society for the Promotion of Science, Chiyoda, Tokyo 102-0083, Japan.}
}

\maketitle
\thispagestyle{empty}
\pagestyle{empty}

\begin{abstract}
We propose an efficient encrypted policy synthesis to develop privacy-preserving model-based reinforcement learning. We first demonstrate that the relative-entropy-regularized reinforcement learning framework offers a computationally convenient linear and ``min-free'' structure for value iteration, enabling a direct and efficient integration of fully homomorphic encryption with bootstrapping into policy synthesis. Convergence and error bounds are analyzed as encrypted policy synthesis propagates errors under the presence of encryption-induced errors including quantization and bootstrapping.
Theoretical analysis is validated by numerical simulations. Results demonstrate the effectiveness of the RERL framework in integrating FHE for encrypted policy synthesis.
\end{abstract}


\section{INTRODUCTION}
\IEEEPARstart{R}{einforcement} Learning (RL) is a useful and widely recognized 
framework for solving an optimal control problem \cite[Ch. 1]{sutton2018reinforcement}.
RL algorithm is \emph{model-free} when an explicit environment model---such as system dynamics or reward structures---is not used, and \emph{model-based} when the model is used (whether the model is known a priori, or learned by interacting with the environment).
Sample efficiency is one of the strengths of model-based RL, making it successful for high-risk applications with expensive real-world data such as robotics and finance.
Despite its sample efficiency, building a high-fidelity, and often large-scale model can be a challenge as developing it requires extensive domain expertise and significant investments; yet, for the same reasons, owning such models can be an advantage against competitors. 

After the model is learned satisfactorily, these expensive models may need high-performance computing resources for subsequent computations, such as planning.
While outsourcing these tasks to existing third-party cloud servers can be an attractive solution for maintenance and scalability reasons, it unveils additional data privacy and security issues, which are frequently overlooked yet critical.
For example, there exists a potential eavesdropping threat on the data being transmitted.
In addition, the cloud could try to steal critical information from the hard-earned proprietary model.

Encrypted control \cite{kim2022comparison, schluter2023brief} is a framework that aims to solve the aforementioned security issues of networked control systems. It integrates various cryptographic computation methods, such as homomorphic encryption (HE) or random affine transformations, into control and decision-making processes.
HE is an encryption scheme that allows direct computation of addition and/or multiplication over encrypted data without decryption.
Building upon its initial concept \cite{Kogiso2015CybersecurityEO}, the development of encrypted control has mainly focused on implementing a pre-synthesized controller over encrypted data, for example, infinite-time-horizon dynamic controllers \cite{Teranishi2021InputOutputHF, schluter2021encrypted, kim2022dynamic, jang2024ring}, model predictive control \cite{ec_mpc1, ec_mpc2}, and cooperative control \cite{ec_cooperative2}.

Compared to control \emph{implementation}, control \emph{synthesis} has been less explored in the literature on encrypted control; nevertheless, control synthesis often involves highly private models and specific constraints, as discussed earlier for model-based RL.
In this context, there have been recent efforts to implement encrypted control for private control \textit{design} beyond implementation, such as solving a quadratic programming \cite{alexandru2020cloud} and tuning PID parameters \cite{schluter2022pidtuning}.
However, in a more focused view on integrating RL framework, we only find encrypted model-free RL \cite{suh2021sarsa, suh2021encrypted} closely related to this work; \cite{dzurkova2024approximated} also ties RL with encrypted control but discusses encrypted implementation of the pre-trained explicit NMPC control law via deep RL algorithm.

One of the main difficulties observed in integrating RL with encrypted control has been the restricted arithmetic flexibility of HE because operations commonly needed in RL, such as $\min$, $\max$, or comparison, cannot be approximated easily using only homomorphic additions and multiplications.
Worse yet, HE schemes generally have a limit on the number of multiplications allowed (see Section~\ref{subsec:CKKS}).
Due to these factors, \cite{suh2021sarsa} and \cite{suh2021encrypted} required intermediate re-encryption of encrypted values, which necessitates constant communication efforts from the client.

This paper aims to serve as the foundation for privacy-preserving model-based RL. 
We consider the client-server architecture given in \cite{suh2021sarsa, suh2021encrypted}, where a resource-limited client wishes to outsource the offline policy synthesis task to a cloud server in a privacy-preserving manner, as the client's learned model is assumed to be privacy-sensitive.
The client has a relatively accommodating timeframe for the completion of the synthesis, enjoying more arithmetic flexibility (by the use of bootstrapping) compared to many online implementation scenarios with strict real-time requirements, where bootstrapping must be underutilized.

This paper makes four main contributions :
\begin{itemize}[leftmargin=*]
    \item We demonstrate that a certain class of RL, which we call relative-entropy-regularized RL (RERL), offers a linear and min/max-free structure that is particularly effective in integrating HE.
    \item We propose \textit{Encrypted RERL} that efficiently synthesizes the control policy over the encrypted model. The efficient algorithm allows the direct use of bootstrapping and does not require communication-heavy intermediate re-encryptions.
    \item We analyze the convergence error bounds for the \textit{Encrypted RERL} under the encryption-induced errors.
    \item We validate the feasibility of the \textit{Encrypted RERL} convergence error analysis using numerical experiments. We distribute an open-source library that includes the experiment in this letter, as well as base homomorphic operations and high-level subroutines such as linear algebra operations and bootstrapping used to experiment, enabling reproducibility and further extensions.
\end{itemize}

\emph{Notation:} The set of integers, positive integers, and complex numbers are denoted by $\bbZ$, $\bbN$, $\bbC$, respectively.
For a (matrix) vector of scalars, $\|\cdot\|$ denotes the (induced) infinity norm.
The Hadamard product (element-wise multiplication) is written by $\odot$.
For $x\in\bbC^n$ and $r\ge 0$, we denote by $\RotVec_r(x)$ the vector obtained by shifting the element of $x$ upward (or left) by $r$ positions. 
For example, $\RotVec_2([1 \ 2 \ 3 \ 4 \ 5]^\top) = [3 \ 4 \ 5 \ 1 \ 2]^\top$.

\section{Preliminaries}\label{prelim}

\subsection{Fully Homomorphic Encryption - CKKS Cryptosystem}\label{subsec:CKKS}
Generally speaking, an HE scheme is a fully homomorphic encryption scheme (FHE) if it supports both homomorphic addition and multiplication for an arbitrary number of times.
Typically, a FHE scheme accompanies bootstrapping \cite{gentry}, a special method to enable an arbitrarily many number of homomorphic operations.
Throughout this letter, we use the CKKS (Cheon-Kim-Kim-Song) cryptosystem\cite{cheon2017homomorphic}, which is a FHE with its bootstrapping \cite{cheon2018bootstrapping}. We briefly review its basic operations, security, and effects of errors in the next.

A message (to be encrypted) lives in $\bbC^{N/2}$, the $N/2$ dimensional (i.e., $N/2$ slots) vector space of complex floating-point numbers.
The message is first encoded into a plaintext; quantization error is introduced during this encoding process as the message turns into integers by multiplying a scaling factor $\Delta>0$ and rounding. 
The encryption---or the encoding--and--encryption\footnote{\emph{Encoding} typically takes place before encryption. However, for simplicity, we make this implicit as a part of the encryption algorithm $\Enc$. Similarly, \emph{decoding} takes place after $\Dec$ but is made implicit.} to be precise---algorithm $\Enc: \bbC^{N/2} \to \calC$ maps the input message to an encrypted message (the ciphertext) that lives in $\mathcal{C}$, the ciphertext space.

The following homomorphic operations can be used in the CKKS cryptosystem to conduct arithmetic on ciphertexts:
\begin{itemize}
    \item $\hphantom{\RotCt_r:\calC  \to \calC} \mathllap{\oplus:\calC \times \calC \to \calC}$ \quad (addition)
    \item $\hphantom{\RotCt_r:\calC  \to \calC} \mathllap{\otimes:\calC \times \calC \to \calC}$ \quad (element-wise multiplication) 
    \item $\RotCt_r:\calC  \to \calC$ \quad ($\RotVec_r$)
\end{itemize}
and can be understood better by examining their effects on decryption. The decryption algorithm $\Dec: \calC \to \bbC^{N/2}$ \emph{approximately} recovers the message encrypted by $\Enc$.
For example, $\oplus$ evaluates the addition of two messages $x_1, x_2 \in \bbC^{N/2}$ as $\Dec(\Enc(x_1)\oplus \Enc(x_2)) \simeq x_1 + x_2$.
The decryption is \emph{approximate}\cite{cheon2017homomorphic} for several reasons: each ciphertext has its quantization error from encoding, and is injected with small errors during $\Enc$, as the security of CKKS relies on the hardness of the Ring Learning With Errors problem \cite{LyubPeik10}.

Notably, although necessary for quantization and security, these errors can grow under homomorphic operations. If these errors grow beyond the limit, the decryption can fail. This necessitates the bootstrapping operation, which, though expensive, resets the accumulated errors in an old ciphertext:
$$\Boot: \calC \to \calC \quad (\text{bootstrapping)}$$

The accumulation of errors under each homomorphic operation discussed earlier can be upper bounded as in Lemma~\ref{lem:homo}; they are presented with Big-O notations for simplicity, but concrete bounds can be found in  \cite{cheon2017homomorphic} and \cite{cheon2018bootstrapping}.
\begin{lem}\label{lem:homo}
    For any $x\in\bbC^{N/2}$, $\bfc,\bfc'\in \calC$, and $r\ge 0$, there exists
$\sfB^{\Enc},\sfB^{\Mult},\sfB^{\RotCt},\sfB^{\Boot}\in \calO(N/\Delta)$ such that 
    \begin{subequations}
        \begin{align}
            \left\|\Dec(\Enc(x)) - x \right\| &\le \sfB^{\Enc}, \label{eq:correct} \\
            \left\|\Dec(\bfc\oplus\bfc') - \left(\Dec(\bfc) + \Dec(\bfc')\right)  \right\|&=0, \label{eq:homoAdd} \\
            \left\|\Dec(\Enc(x)\otimes \bfc) - x\odot \Dec(\bfc)\right\| & \le \sfB^{\Mult}, \label{eq:homoMult}\\
            \left\|\Dec(\RotCt_r(\bfc)) - \RotVec_r(\Dec(\bfc))\right\| & \le \sfB^{\RotCt}, \label{eq:homoRot}\\
            \left\|\Dec(\Boot(\bfc)) - \Dec(\bfc)\right\| & \le \sfB^{\Boot}.\label{eq:boot}
        \end{align}
    \end{subequations}
\end{lem}

Here, $\bfc$ and $\bfc'$ in Lemma~\ref{lem:homo} are not necessarily the fresh outputs from $\Enc$; they could be the output ciphertexts on which a composition of many homomorphic operations is applied. Lemma~\ref{lem:homo} states that the error growth attributed to each low-level homomorphic operation is bounded by some constants. They will help break down the total error analysis (Section~\ref{subsec:convergence}) for the high-level encrypted algorithm, which consists of multiple low-level homomorphic operations.

The parameters $N$ and $\Delta$ affect the security of the CKKS cryptosystem; roughly, increasing $N$ or decreasing $\Delta$ enhances the security level. In practice, these parameters must be chosen carefully under the security and precision requirements; however, choosing practical security parameters is beyond the scope of this letter; we refer to \cite[Sec. 5]{cheon2017homomorphic}.

\subsection{Model-Based Reinforcement Learning}\label{subsec:mbrl}
RL\cite{sutton2018reinforcement} can be formalized by the Markov Decision Processes (MDP) framework.
A discrete-time MDP with time-invariant dynamics and a cost function is formalized by the tuple $(\mathcal{X}, \mathcal{U}, P, C)$, where $\mathcal{X}$ is the finite state space, $\mathcal{U}$ is the finite action space, 
$P(\cdot|x,u):\mathcal{X}\to [0, 1]$ is the state transition probability distribution governing the system dynamics when input $u$ is applied at state $x$,
and $C:\mathcal{X}\times\mathcal{U}\rightarrow\mathbb{R}$ is the cost incurred during state transition.

A policy $\pi = (\pi_{0}, \pi_{1}, \pi_{2}, \cdots)$ is a sequence of probability distributions over actions, where $\pi_{t}(u_t|x_{0:t}, u_{0:t-1})$ gives the probability of choosing an action $u_t\in\mathcal{U}$ conditioned on the history $(x_{0:t}, u_{0:t-1})$.
For a time-invariant and discounted MDP, an optimal policy of the form $\pi(u_t|x_t)$, that is time-invariant and Markov, always exists \cite{puterman2014markov}. 
Without loss of optimality, we assume the optimal policy to be deterministic in the sequel, i.e., $u_t = \pi(x_t)$ for all $x_t\in\mathcal{X}$.

The expected cumulative cost of a policy $\pi$ can be quantified by the value function $V^{\pi}: \mathcal{X} \rightarrow \mathbb{R}$:
\begin{equation}
\label{eq:value_pi_expanded}
V^\pi(x)=\mathbb{E}^\pi[C(x_0, u_0) + \textstyle\sum_{t=1}^{\infty}\gamma^{t}C(x_t, u_t) \mid x_0=x],
\end{equation}
where $C(x_0, u_0)$ is the immediate cost for the initial state $x_0=x$, and $\gamma \in [0, 1)$ is the discount applied to the future cost starting from the next state $x_{1} \in \mathcal{X}$.
For deterministic policies, the value function
\eqref{eq:value_pi_expanded} admits a recursive form:
\begin{equation}
\label{eq:value_pi_recursion}
V^\pi(x) =C(x, u) + \gamma \textstyle\sum_{x'\in \mathcal{X}}P(x' | x, u)V^{\pi}(x').
\end{equation}

The optimal policy $\pi^{*}$ minimizes \eqref{eq:value_pi_recursion} for all $x \in \mathcal{X}$, and we denote the corresponding optimal value function by $V^{*}$ 
This leads to the celebrated Bellman's principle of optimality:
\begin{equation*}
\label{eq:bellman_optimality}
V^{*}(x)=\textstyle\min_{u \in \mathcal{U}}[C(x, u) + \gamma \textstyle\sum_{x'\in \mathcal{X}}P(x' | x, u)V^{*}(x')],
\end{equation*}
and
$\pi^{*}$ can be synthesized as the minimizer
\begin{equation*}
\label{eq:optimal_policy}
\pi^*(x)=\textstyle\argmin_{u\in\mathcal{U}} [C(x,u)+\gamma\! \textstyle\sum_{x'\in\mathcal{X}} P(x'|x,u)V^*(x')].
\end{equation*}
Policy synthesis aims to learn the optimal policy $\pi^{*}$. 

\emph{Value iteration} (VI) is an iterative algorithm that solves for the optimal value $V^{*}$ given an arbitrary initial value $V_{0}(x)$:
\begin{equation}
\label{eq:value_iteration}
V_{k+1}(x)=\textstyle\min_{u \in \mathcal{U}}  [C(x, u) + \gamma \textstyle\sum_{x'\in \mathcal{X}}P(x' | x, u)V_{k}(x')], \notag
\end{equation}
and the convergence of VI is guaranteed after a finite number of iterations when the state and action spaces are finite.

The client can consider outsourcing the above VI to a resourceful cloud server. However, the client risks unwanted disclosure of its private model. While HE can provide privacy protection, it introduces a technical challenge for the client: homomorphically evaluating the $\min$ function using only additions and a limited number of multiplications.

Though approximating $\min$ is possible, accurate approximations require a direct evaluation of high-degree polynomials, or a separate iterative algorithm \cite{hecomparison}, consuming many multiplications for each iteration of VI, making this approach inefficient. This motivates us to propose an alternative framework in the sequel, which leads to an efficient algorithm entirely without the need for evaluating the $\min$ itself.

\section{Main Results: Encrypted Policy Synthesis} \label{main}
This section first presents the relative-entropy-regularized RL (RERL) framework, closely related to linearly-solvable MDP \cite{todorov2009efficient},
and derives a linear and $\min$-free VI. This HE-friendly VI resolves the limitation noted towards the end of Section~\ref{subsec:mbrl}, and leads to an efficient encrypted algorithm, which we call \textit{Encrypted RERL} in Section~\ref{encrypted_RERL}.

\subsection{RERL Framework and HE-Friendly Value Iteration}
Suppose we are given a nominal probability distribution $b(u|x)$, which can be understood as the default policy of an agent. In RERL, the optimal policy is not necessarily deterministic---therefore, we optimize over the space of time-invariant, Markov, and randomized policies $\pi(u|x)$ in the sequel. The stage-wise cost to be minimized 
by the policy $\pi(u|x)$ 
is
\begin{equation}\label{eq:aug_cost}
    C_{\text{aug}}(x, u) = C(x, u) + \lambda \log \frac{\pi(u|x)}{b(u|x)},
\end{equation}
where we assume $C(x, u)\geq 0$ for each state-action pair $(x, u)$ and $\lambda > 0$ is a regularization constant.
When taking the expectation of $C_{\text{aug}}(x, u)$ with respect to actions sampled from the policy $\pi$---in the similar context of the expected cumulative cost \eqref{eq:value_pi_expanded}, the second term becomes 
$\textstyle\sum_{u \in \mathcal{U}}\pi(u|x)\log\textstyle\frac{\pi(u|x)}{b(u|x)}$, 
which corresponds to the definition of RE of $\pi(\cdot|x)$ from $b(\cdot|x)$. Thus, the second term can be thought of as a penalty for deviations from the default policy $b(u|x)$. We also impose that $\pi(u|x) = 0$ for all $u\in \mathcal{U}$ such that $b(u|x) = 0$.

To exploit the computational advantages of RERL, we make the following assumptions:
\begin{enumerate}[label=(A-\arabic*), leftmargin=*, itemindent=1em, labelsep=0.5em]
    
    \item The state transition $P(x_{t+1}|x_{t}, u_t)$ is deterministic. In other words, there exists a deterministic function $F$ such that $x_{t+1} = F(x_t, u_t)$. \label{assume:1}
    \item  \label{assume:2} The cost to be minimized is the following expected \emph{undiscounted} cumulative cost:
    \begin{equation}\label{eq:value_pi_modified}  V^\pi(x)=\mathbb{E}^\pi[\textstyle\sum_{t=0}^{\infty}C_{\text{aug}}(x_t, u_t) \mid x_0=x].
    \end{equation}   
    To ensure that \eqref{eq:value_pi_modified} is bounded, we additionally assume
    \item There exists at least one absorbing state $x_{\text{abs}} \in \mathcal{X}$ such that $x_{\text{abs}}$ is reachable from any starting state $x_0\in\mathcal{X}$. Moreover, 
    $C(x, u) > 0$ for all $x \in \mathcal{X} \setminus\{x_{\text{abs}}\}$ and $u\in\mathcal{U}$ while $C(x_{\text{abs}}, u) = 0$ and $\pi(u|x_{\text{abs}}) = b(u|x_{\text{abs}})$ for all $u \in \mathcal{U}$
    and the RERL value function $V^{\pi}(x_{\text{abs}}) = 0$ for all $\pi$.
    \label{assume:3}   
\end{enumerate}
While these assumptions can seem relatively restrictive, many real-world problems can be formulated within this framework~\cite{todorov2009efficient}.
More importantly, these assumptions allow computational advantages as we will show soon.

\begin{lem}\label{lem:RERL_optimal}
    Under assumptions (A-2) and (A-3), the RERL optimal value function $V^*(x)$ is bounded and satisfies
    \begin{equation}
    \label{eq:lem2_vstar}
    V^*(x)=-\lambda \log \left(\sum\nolimits_{u\in\mathcal{U}} b(u|x)\exp\left(-\rho(x,u)/\lambda\right)\right)
    \end{equation}
    where $\rho(x, u) := C(x, u) + \sum_{x'\in\mathcal{X}} P(x'|x,u)V^*(x')$. Moreover, the RERL optimal policy $\pi^{*}$ takes the form of a Boltzmann distribution:
    \begin{equation}\label{eq:pi_star}
        \pi^*(u|x)=\frac{b(u|x)\exp(-\rho(x,u)/\lambda)}{\sum_{u'\in\mathcal{U}}b(u'|x)\exp(-\rho(x,u')/\lambda)}.
    \end{equation}
\end{lem}
\begin{proof}
By Bellman's optimality principle, $V^*(x)$ satisfies 
\begin{align}
V^*(x)&=\min_\pi \sum_{u\in\mathcal{U}} \pi(u|x) \!\left(\!C_{\text{aug}}(x,u)+\!\!\sum_{x'\in\mathcal{X}}\!P(x'|x,u)V^*(x')\!\right) \nonumber \\
&=\min_\pi \sum_{u\in\mathcal{U}} \pi(u|x) \left(\rho(x,u)+\lambda\log\frac{\pi(u|x)}{b(u|x)}\right). \label{eq:lem2_min}
\end{align}
The minimizer of the convex optimization \eqref{eq:lem2_min} is given by \eqref{eq:pi_star} -- see \cite[Proposition 1.4.2]{dupuis2011weak} for an elementary proof. The result \eqref{eq:lem2_vstar} is obtained by substituting \eqref{eq:pi_star} into \eqref{eq:lem2_min}.
\end{proof}
The following theorem emphasizes the linearly solvable nature of RERL.
\begin{thm}\label{thm:linear_z}
Suppose assumptions (A-1), (A-2), and (A-3) hold. Let $V^*(x)$ be the RERL optimal value function, and define the desirability function $z^*(x)$ via the exponential transformation $z^*(x):=\exp(-V^{*}(x)/\lambda)$.
Then, $z^*(x)$ satisfies the linear equation   
\begin{equation} \label{eq:z_linear}
    z^*(x) =  \textstyle\sum_{u\in\mathcal{U}} b(u|x)\exp\left(-C(x,u)/\lambda\right) z^*(F(x,u)).
\end{equation}
The optimal policy can be expressed in terms of $z^*(x)$ as
    \begin{equation}\label{eq:recon}
    \pi^*(u|x) = \frac{b(u|x)\exp(-C(x, u)/\lambda) z^{*}(F(x, u))}
                    {z^{*}(x)}.
    \end{equation}
\end{thm}
\begin{proof}
Applying the exponential transformation to \eqref{eq:lem2_vstar}, we obtain
\begin{align*}
&\exp(-V^*(x)/\lambda)  \\
&=\sum_{u\in\mathcal{U}}b(u|x)\exp\left(\!-\frac{C(x,u)}{\lambda}\!+\!\sum_{x'\in\mathcal{X}}P(x'|x,u)\left(\!-\frac{V^*(x')}{\lambda}\!\right)\!\right)  \\
&\leq \sum_{u\in\mathcal{U}}\sum_{x'\in\mathcal{X}}P(x'|x,u)b(u|x)\exp\left(-\frac{C(x,u)}{\lambda}-\frac{V^*(x')}{\lambda}\right) 
\end{align*}
where the last step follows from Jensen's inequality.
Under the assumption (A-1) of deterministic transition, the inequality holds with equality, yielding
\begin{align*}
&\exp(-V^*(x)/\lambda)  \\
&=\sum_{u\in\mathcal{U}}b(u|x)\exp\left(-\frac{C(x,u)}{\lambda}\right)\exp\left(-\frac{V^*(F(x,u))}{\lambda}\right). 
\end{align*}
This proves \eqref{eq:z_linear}. The result \eqref{eq:recon} follows from \eqref{eq:pi_star}.
\end{proof}

Importantly, Theorem~\ref{thm:linear_z} leads to an efficient VI algorithm that is linear and free of $\min$. To see this, let us first assume, without loss of generality,  a single absorbing state $x_{\text{abs}}$, and enumerate the state space as $\mathcal{X}=\{q_1, q_2, \cdots, q_{S}, x_{\text{abs}}\}$, where $S$ is the cardinality of the non-absorbing states.
Then, we construct a vector that consists of the $z^*$ values of non-absorbing states as follows:
\begin{equation}\label{z_vec}
    Z^{*} = \begin{bmatrix}
        z^{*}(q_{1}) & z^{*}(q_{2}) & \cdots & z^{*}(q_{S})
    \end{bmatrix}^{\top}.
\end{equation}
As each element of \eqref{z_vec} satisfies \eqref{eq:z_linear}, we have the following:
\begin{equation} \label{eq:z_star_linear}
    Z^{*} = AZ^{*} + w,
\end{equation}
where $A \in \mathbb{R}^{{S}\times S}$ and $w \in \mathbb{R}^{S}$ are defined element-wise by
\begin{align}
    a_{ij} &:= \textstyle\sum_{u: F(q_i,u) = q_j} b(u|q_i)\exp\left(-C(q_i,u)/\lambda\right), \label{eq:A_elem}\\
    w_{i} & := \textstyle\sum_{u: F(q_i,u) = x_{\text{abs}}} b(u|q_i)\exp\left(-C(q_i,u)/\lambda\right).
\end{align}
Note that we excluded the absorbing state since $z^*(x_{\text{abs}}) = 1$ by (A-3), and iteration is unnecessary.

\begin{lem}\label{lem:contract}
    The spectral radius $r(A)$ of the square matrix $A$ defined by \eqref{eq:A_elem} satisfies $r(A) < 1$.
\end{lem}
\begin{proof}
We have $a_{ij} \geq 0$ for all $i,j$, and thus, the square matrix $A$ is non-negative. Then, by \cite[Theorem 8.1.22]{horn2012matrix}:
\begin{equation}
    \textstyle \min_i \sum_j a_{ij} \leq r(A) \leq \max_i \sum_j a_{ij}.\notag
\end{equation}
In fact, each row sum of $A$ satisfies
\begin{align}
    \textstyle\sum_{j} a_{ij} \!&=\! \textstyle\sum_{j} \sum_{u: F(q_i,u) = q_j} b(u|q_i)\exp\left(-C(q_i,u)/\lambda\right) \nonumber, \\
    \!& =\! \textstyle\sum_{u: F(q_i,u) \in \mathcal{X}\setminus \{x_{\text{abs}}\}} b(u|q_i)\exp\left(-C(q_i,u)/\lambda\right) \!<\! 1 \nonumber,
\end{align}
where the inequality holds because of Assumption (A-2) and (A-3). Therefore, $r(A) < 1$ follows.
\end{proof}

Lemma~\ref{lem:contract} implies that $A$ is contractive and that we have a linear and stable recursion:
\begin{equation}
    \label{eq:lin_sys}
    Z_{k+1} = A Z_{k} + w,
\end{equation}
that converges to ${Z}^{*}$ for any $Z_{0} > 0$. 

\subsection{Encrypted RERL Policy Synthesis}\label{encrypted_RERL}

We are now ready to present \textit{Encrypted RERL}. To proceed, let us first transform the linear system \eqref{eq:lin_sys} into the \emph{encryption-friendly} version using addition, Hadamard product $\odot$, and vector-rotation $\RotVec_r$, as follows:
\begin{equation}\label{eq:nominal}
    \begin{split}
        g_{k,i} &= \textstyle\sum_{r=0}^{S-1} \RotVec_{r}(A_i \odot Z_k),  \\
        Z_{k+1} &=w + \textstyle\sum_{i=1}^S e_i \odot g_{k,i},
    \end{split}
\end{equation}
where $A_i^\top\in\bbC^{1\times S}$ denotes the $i$-th row vector of $A$ and $e_i\in\bbC^{S}$ is the vector with $1$ in the $i$-th component and $0$'s elsewhere for $i=1,\ldots,S$.
It can be seen that $g_{k,i}$ is a vector with elements holding copies of the inner product $A_i^\top Z_k$.

Using operations from Section~\ref{subsec:CKKS}, we can now encrypt \eqref{eq:nominal}, which yields the primary equation\footnote{If we have $S < N/2$ in \eqref{eq:lin_sys}, we can always zero-pad the matrices and vectors until $S=N/2$ to make encryptions in \eqref{eq:enc_RERL} well-defined.} of \textit{Encrypted RERL}:
    \begin{equation}\label{eq:enc_RERL}
        \begin{split}
        \bfg_{k,i} &= \textstyle\sum_{r=0}^{S-1}\RotCt_r(\Enc(A_i)\otimes \bfz_k), \\
        \bfz_{k+1} &= \Boot(\Enc(w) \oplus \textstyle\sum_{i=1}^S \Enc(e_i) \otimes \bfg_{k,i} ),
        \end{split}
    \end{equation}
where the initial value is encrypted as $\bfz_0 = \Enc(Z_0)$ and the summations are taken with respect to the homomorphic addition $\oplus$. 

These two operations of \eqref{eq:enc_RERL} can be executed by the server, who uses the encryptions of $Z_0$, $w$, and $A_i$ and $e_i$ for $i=1,\ldots,S$ (the encrypted model) received from the agent as in Fig~\ref{fig:0}.
\begin{figure}
    \centering  
    \includegraphics[trim=18pt 18pt 18pt 18pt, clip, scale=0.16]{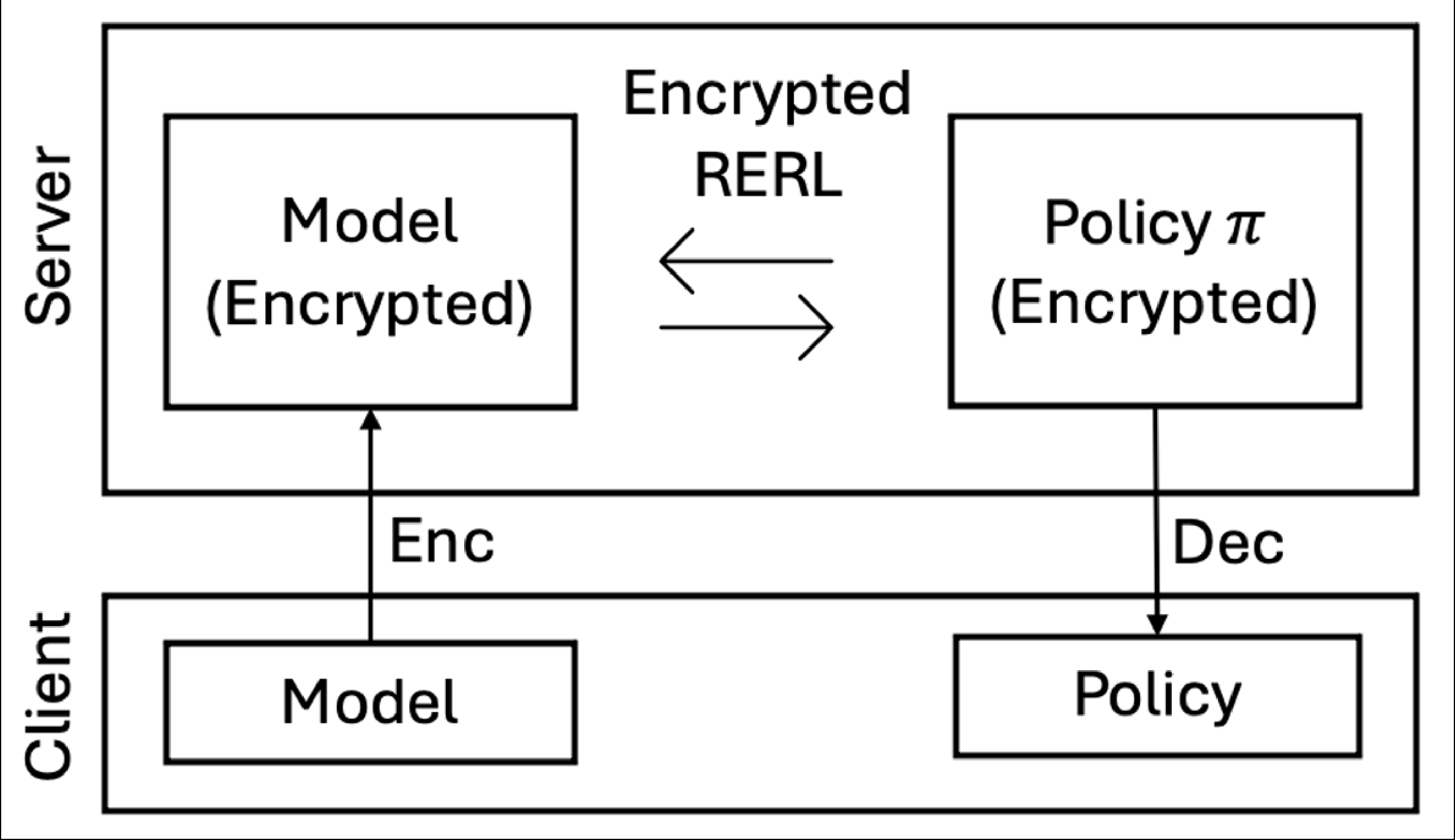}
    \caption{Outsourced \textit{Encrypted RERL}: an agent encrypts the model and outsources to the server offline. The server runs \eqref{eq:enc_RERL} over FHE and returns the result, which the agent can decrypt to construct the policy.}
    \label{fig:0}
    \vspace{-5mm}
\end{figure}
After a pre-determined number of iterations, i.e., $k=1, \ldots,T\in\bbN$, the server returns $\bfz_T$ to the client, who decrypts to obtain $\tilde{Z}_{T} = \Dec(\bfz_{T})$, and the client can reconstruct the policy using \eqref{eq:recon}.

\subsection{Convergence and Error Analysis}\label{subsec:convergence}
In what follows, we analyze the performance of the proposed \textit{Encrypted RERL} \eqref{eq:enc_RERL}. 
To this end, we define 
\begin{equation*}
    \tilde{Z}_k := \Dec(\bfz_k), \quad \forall k\ge 0,
\end{equation*}
and show that $\tilde{Z}_k$ converges to a vicinity of the optimal $Z^{*}$ value.
We derive an explicit upper bound function on the convergence error $\|\tilde{Z}_k-Z^*\|$.
The following lemma first states that $\tilde{Z}_k$ follows a dynamics of the form \eqref{eq:lin_sys} with a bounded perturbation.

\begin{lem}
\label{lem:error_bound}
    The dynamics of $\tilde{Z}_k$ is given by
    \[
        \tilde{Z}_{k+1} = A \tilde{Z}_k + w + \epsilon_k, \quad \tilde{Z}_0 = Z_0 + \epsilon_0,
    \]
    where $\| \epsilon_0 \| \le \beta_0$ and $\| \epsilon_k \| \le \beta$ for all $k\ge 0$
    for some positive constants $\beta_0 \in \calO(N / \Delta)$ and $\beta \in \calO(SN / \Delta)$.
\end{lem}

\begin{proof}
    By definition, $\tilde{Z}_0 = \Dec(\Enc(Z_0))$, and therefore, $\beta_0:=\sfB^{\Enc} \in \calO(N/\Delta)$ by \eqref{eq:correct}.
    In order to derive the upper bound $\beta$, observe that for each $k\ge 0$ and $i=1,\ldots,S$,
    \begin{align*}
        &\Dec(\bfg_{k,i}) = \textstyle\sum_{r=0}^{S-1} \Dec(\RotCt_r(\Enc(A_i)\otimes \bfz_k)),\! \tag{$\because$ \eqref{eq:homoAdd}} \\
            &= \textstyle\sum_{r=0}^{S-1}\RotVec_r(\Dec(\Enc(A_i)\otimes \bfz_k)) + \epsilon^\Rot_r, \tag{$\because$ \eqref{eq:homoRot}} \\
            &= \textstyle\sum_{r=0}^{S-1} \RotVec_r(A_i\odot \tilde{Z}_k+\epsilon^\Mult ) +\epsilon^\Rot_r, \tag{$\because$ \eqref{eq:homoMult}} \\
            &=: \!\!\underbrace{\textstyle\sum_{r=0}^{S-1} \RotVec_r(A_i\odot \tilde{Z}_k)}_{=:\tilde{g}_{k,i}}  \!+ \underbrace{\!\textstyle\sum_{r=0}^{S-1} \RotVec_r(\epsilon^\Mult ) \!+\epsilon^\Rot_r}_{=:\epsilon^g_{k,i}},
    \end{align*}
    for some $\epsilon^\Mult, \epsilon^\Rot_r \in \bbC^{S}$ bounded by $\|\epsilon^\Mult \| \le \sfB^\Mult$ and $\| \epsilon^\Rot_r\|\le \sfB^\RotCt$ for $r=0,\ldots,S-1$, which leads to
    \begin{equation*}
        \|\epsilon^g_{k,i} \| \le S\cdot\left(\sfB^\Mult + \sfB^\RotCt \right). 
    \end{equation*}
    Then, for each $k\ge 0$,
    \begin{align*}
        &\tilde{Z}_{k+1} = \Dec(\Boot(\Enc(w) \oplus \textstyle\sum_{i=1}^S \Enc(e_i) \otimes \bfg_{k,i} )), \\
        &= \Dec(\Enc(w) \oplus \textstyle\sum_{i=1}^S \Enc(e_i) \otimes \bfg_{k,i}) + \epsilon_k^\Boot, \tag{$\because$ \eqref{eq:boot}}  \\
        &= w \!+\!\! \textstyle\sum_{i=1}^S\! \Dec(\Enc(e_i) \!\otimes\! \bfg_{k,i}\!) \!+\! \epsilon_k^{\Enc} \!+\! \epsilon_k^\Boot\!, \!\!\tag{$\because$ \!\eqref{eq:correct}, \!\eqref{eq:homoAdd}} \\
        &= w \!+\! \textstyle\sum_{i=1}^S e_i \!\odot\! \left(\!\tilde{g}_{k,i} \!+\! \epsilon^g_{k,i}\! \right) \!+\!\epsilon^\Mult_{k,i}\!+\! \epsilon_k^{\Enc} \!+\! \epsilon_k^\Boot\!, \tag{$\because$ \eqref{eq:homoMult}} \\
        &= A\tilde{Z}_k + w + \underbrace{\textstyle\sum_{i=1}^S e_i \odot  \epsilon^g_{k,i} + \epsilon^\Mult_{k,i} + \epsilon_k^{\Enc} + \epsilon_k^\Boot}_{=:\epsilon_k} ,
    \end{align*}
    for some $\epsilon^\Mult_{k,i}, \epsilon_k^{\Enc}, \epsilon_k^\Boot \in \bbC^{S}$ bounded by $\| \epsilon^\Mult_{k,i}\|\le \sfB^\Mult$ for $i=1,\ldots,S$, $\|\epsilon_k^{\Enc} \| \le \sfB^{\Enc}$, and $\|\epsilon_k^\Boot\|\le \sfB^\Boot$. 
    Note that the last equality follows from \eqref{eq:nominal} and the definition of $\tilde{g}_{k,i}$.
    As a result, we have
    \begin{equation*}
        \|\epsilon_k \| \le S\cdot\left(2\sfB^\Mult + \sfB^\RotCt \right) + \sfB^{\Enc} + \sfB^\Boot \in \calO(SN/\Delta),
    \end{equation*}
    and this concludes the proof.
\end{proof}

It can be seen that the perturbations $\beta_0$ and $\beta$ in Lemma~\ref{lem:error_bound} originate from the growth of errors under homomorphic operations.
The following theorem states that the effect of those perturbations on the convergence error $\|\tilde{Z}_k - Z^* \|$ remains bounded for each $k\ge 0$.

\begin{thm}\label{thm:error}
    There exist some constants $c \ge 1$ and $\alpha \in (0, 1)$ such that 
    \[
        \| \tilde{Z}_k - Z^* \| < c \left( \alpha^k ( \| Z_0 - Z^* \| + \beta_0 ) + \frac{ \beta }{ 1 - \alpha } \right),
    \]
    where $\beta_0$ and $\beta$ are as in Lemma~\ref{lem:error_bound}.
\end{thm}

\begin{proof}
    If holds from   Lemma~\ref{lem:error_bound} that
    \[
        \tilde{Z}_k = A^k \tilde{Z}_0 + \textstyle\sum_{s=0}^{k-1} A^s w + \textstyle\sum_{s=0}^{k-1} A^s \epsilon_{k-1-s}, ~~ \tilde{Z}_0 = Z_0 + \epsilon_0,
    \]
    where $\| \epsilon_0 \| \le \beta_0$ and $\| \epsilon_{k-1-s} \| \le \beta$.
    In addition, the contractivity of $A$ (Lemma~\ref{lem:contract}) implies that there exist $c \ge 1$ and $\alpha \in (0, 1)$ such that $\| A^k \| \le c \alpha^k$ for all $k \ge 0$.
    Therefore, it follows that
    \begin{align*}
        \| \tilde{Z}_k - Z^* \|
        &\le \| A^k \| \| \tilde{Z}_0 - Z^* \| + \textstyle\sum_{s=0}^{k-1} \| A^s \| \| \epsilon_{k-1-s} \|, \\
        &< c \alpha^k ( \| Z_0 - Z^* \| + \beta_0 ) + \textstyle\sum_{s=0}^\infty c \alpha^s \beta,
    \end{align*}
    where $Z^* = A Z^* + w = A^k Z^* + \sum_{s=0}^{k-1} A^s w$.
    The claim holds because $\sum_{s=0}^\infty c \alpha^s = c / (1 - \alpha)$.
\end{proof}

Since $\alpha\in(0,1)$, it holds from Theorem~\ref{thm:error} that 
\begin{equation*}
    \limsup_{k\to \infty} \|\tilde{Z}_k - Z^* \| < c\frac{\beta}{1-\alpha},
\end{equation*}
where the right-hand-side can be made arbitrarily small by increasing the scale factor $\Delta$ relative to the degree $N$. However, as discussed in Section~\ref{subsec:CKKS}, for security reasons, increasing $\Delta$ in turn needs increasing $N$, which can slow down the computation speed.

\section{Numerical Experiment}\label{experiment}
We have performed a numerical experiment to: (i) test the feasibility and correctness of the \textit{Encrypted RERL}, and (ii) to validate the analysis on encryption-induced errors, and the computation time based on the CKKS parameters.

The experiment was set up for the \emph{Grid-World} environment, a simple toy example frequently employed for tabular RL for finite state MDPs. Following parameters were held constant after arbitrarily chosen: $\lambda = 10.0$, $C(x, u)=0.5$ (the stage cost) for all non-absorbing states and action pairs, $|\mathcal{U}| = 9$ for horizontal, vertical, and diagonal directional movements as well as do-nothing.
The state space size $|\mathcal{X}|$ was constrained by the parameter $N$ as we wanted each $Z$ to have a dimension less than or equal to the slot size $N/2$. For example, for $N=2^3$, the slot size needs to be $N/2 = 4$, and thus $|\mathcal{X}|=4$ (including one absorbing state). However, we expect that parallelization can resolve this constraint.
An agent's default policy was set as a uniform distribution over possible actions, i.e., $b(u|x) = 1/|\mathcal{U}|$.

\begin{figure}
    \centering  
    \includegraphics[trim=5pt 5pt 5pt 5pt, clip, width=\columnwidth]{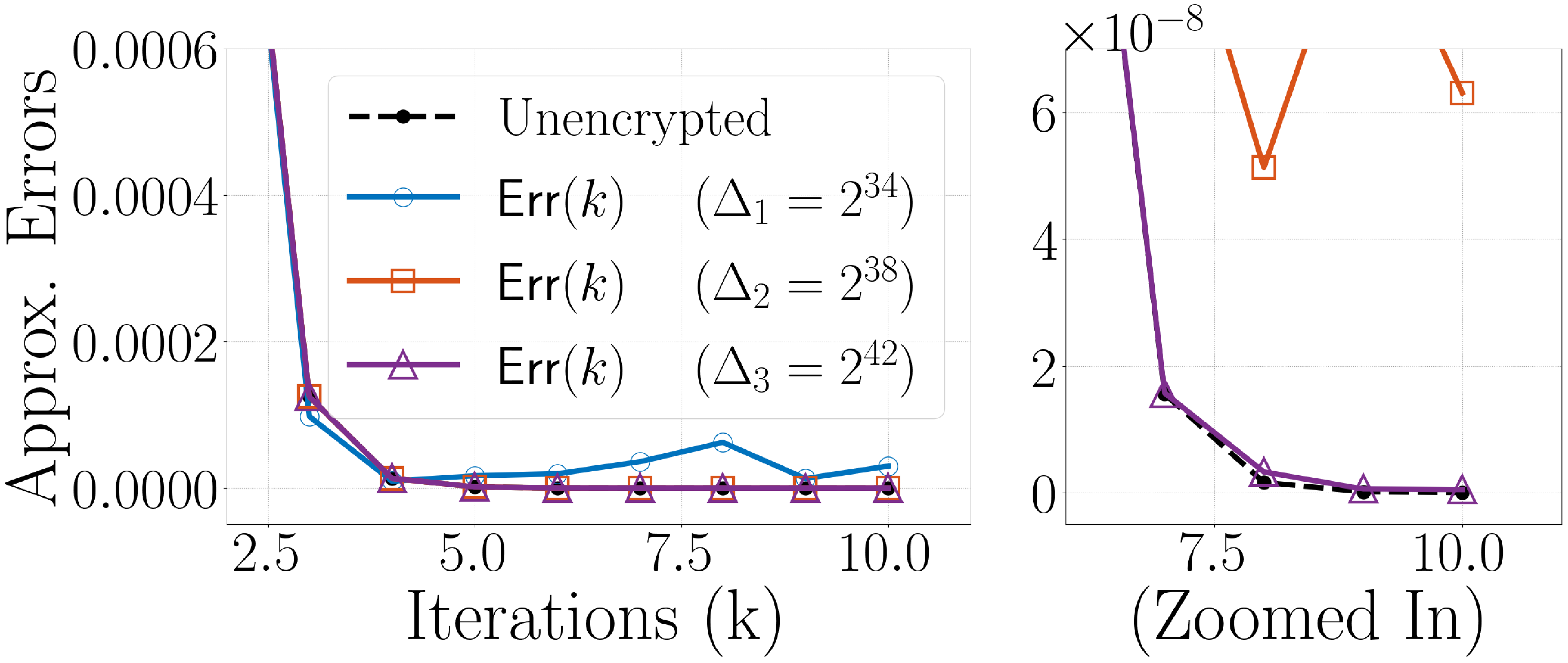}
    \vspace{-5mm}
    \caption{Encryption errors vs. scaling factors $\Delta_{1,2,3} = (2^{34}, 2^{38}, 2^{42})$.}
    \label{fig:1}
    \vspace{-5mm}
\end{figure}

For the experiment, we first computed the optimal value $Z^{*}$ as a reference using \eqref{eq:z_star_linear}. Upon observing that the unencrypted VI of \eqref{eq:lin_sys} converged rapidly (less than $T=30$) under the convergence tolerance $1\mathrm{e}{-10}$, we set the max iteration at $T=50$ to ensure a reasonable number of iterations for the encrypted counterpart. At each iteration, we measured the computation time and the encryption error defined by 
$$\mathsf{Err}(k) := \frac{\mathbb{E}_{q_{i}\in \mathcal{X} \setminus \{x_{\text{abs}}\}}[|\tilde{Z}_{k} - Z^{*}|]}{\mathbb{E}_{q_{i}\in \mathcal{X} \setminus \{x_{\text{abs}}\}}[Z^{*}]}.$$
For reference, Config. 3 has the following sample values:
\begin{align}
    Z^* &= 1.00\mathrm{e}{-2} \begin{bmatrix} 2.4295 & 4.1524 & 3.5592 & 3.7762 & \cdots \end{bmatrix} \nonumber\\
    \tilde{Z}_{T} &= 1.00\mathrm{e}{-2} \begin{bmatrix} 2.4290 & 4.1518 & 3.5588 & 3.7758 & \cdots \end{bmatrix} \nonumber
\end{align}

\begin{table}[t]
    \centering
    \caption{Computation time statistics (Mean, Min, and Max) for each iteration of \eqref{eq:enc_RERL} and encryption error at the final iteration.}
    \setlength{\tabcolsep}{3pt}
    \begin{tabular}{c|*{3}{c}|*{3}{c}*{1}{c}}
        \toprule
        \multirow{2}{*}{\textbf{Config.}} & \multicolumn{3}{c|}{\textbf{Parameters}} & \multicolumn{4}{c}{\textbf{Results}} \\
        & $S$ & $N$ & $\Delta$ & Mean (\SI{}{s}) & Min (\SI{}{s}) & Max (\SI{}{s})  & $\mathsf{Err}(T)$ \\
        \midrule        
        1 & $3$ & $2^7$ & $2^{28}$ & $8.05$ & $5.79$ & $8.35$ & $6.04\mathrm{e}{-4}$ \\
        2 & $3$ & $2^7$ & $2^{30}$ & $8.19$ & $5.82$ & $8.91$ & $1.05\mathrm{e}{-4}$ \\
        3 & $7$ & $2^7$ & $2^{28}$ & $32.31$ & $19.25$ &  $37.44$ & $1.32\mathrm{e}{-3}$ \\
        4 & $7$ & $2^7$ & $2^{32}$ & $32.34$ & $20.49$ & $37.83$ & $3.54\mathrm{e}{-4}$ \\
        5 & $3$ & $2^8$ & $2^{29}$ & $41.75$ & 
        $29.06$ & $43.05$ & $6.63\mathrm{e}{-4}$ \\
        6 & $3$ & $2^{10}$ & $2^{30}$ & $2027$ & $1561$ &  $2097$ &  $4.31\mathrm{e}{-3}$ \\
        \bottomrule
    \end{tabular}
    \label{tab:1}
    \vspace{-5mm}
\end{table}
Note that the ciphertext modulus was set at $q=2^{64}$ for each Config. Results in Table~\ref{tab:1} numerically support the theory as (i) increasing $N$ or $S$ increases encrypted computation times; compare Configs. $2$--$6$, or $1$--$3$, and (ii) increasing $\Delta$ generally reduces encryption errors per step; compare Configs. $1$--$2$, or $3$--$4$. This supports Lemma~\ref{lem:error_bound} and Theorem~\ref{thm:error}, which suggested error bounds in $\calO(SN / \Delta)$ per step, asymptotic convergence to a vicinity of the optimal value, and finally, that the error bound could be made smaller by increasing $\Delta$, as visualized in Fig.~\ref{fig:1}.
The results also agree with intuition as both $S$ and $N$ determine the size of the computation at each iteration, for example, $S$ directly affects the number of encrypted operations (see \ref{encrypted_RERL}).
Lastly, as $N$ is a security parameter, it suggests a trade-off between the increased security level vs. computation time. 
The practical implication is that we should increase $\Delta$ relative to $S$ and $N$ to achieve better precision. The impact on practical security is beyond the scope of this study and left to future analysis.

Data obtained, and the complete implementation of the algorithms described in this paper---including low-level encryption utilities and high-level encrypted algorithms we used---are available at \url{https://github.com/jsuh9/HERL}
and results can be reproduced.

\section{Conclusion} \label{final}

We proposed \textit{Encrypted RERL}---a homomorphically encrypted policy synthesis algorithm---as a step towards privacy-preserving model-based RL.
We observed that the RERL framework can eliminate the need for evaluating the $\min$ operation, enabling an efficient encrypted RL over FHE. The effects of encryption errors in the proposed algorithm were theoretically analyzed. A software library has been developed and publicly released to support numerical experiments, and results validated the feasibility and analysis.

For future work, we aim to investigate the class of efficient encrypted RL algorithms beyond tabular RL in the RERL framework. Moreover, performance optimization in terms of underlying encrypted operations can be pursued to reduce the gap between unencrypted and encrypted computations.

\addtolength{\textheight}{-12cm}   




\bibliographystyle{IEEEtran}
\bibliography{IEEEabrv, encRLAbbrev} 

\begin{thebibliography}{10}
\providecommand{\url}[1]{#1}
\csname url@samestyle\endcsname
\providecommand{\newblock}{\relax}
\providecommand{\bibinfo}[2]{#2}
\providecommand{\BIBentrySTDinterwordspacing}{\spaceskip=0pt\relax}
\providecommand{\BIBentryALTinterwordstretchfactor}{4}
\providecommand{\BIBentryALTinterwordspacing}{\spaceskip=\fontdimen2\font plus
\BIBentryALTinterwordstretchfactor\fontdimen3\font minus \fontdimen4\font\relax}
\providecommand{\BIBforeignlanguage}[2]{{%
\expandafter\ifx\csname l@#1\endcsname\relax
\typeout{** WARNING: IEEEtran.bst: No hyphenation pattern has been}%
\typeout{** loaded for the language `#1'. Using the pattern for}%
\typeout{** the default language instead.}%
\else
\language=\csname l@#1\endcsname
\fi
#2}}
\providecommand{\BIBdecl}{\relax}
\BIBdecl

\bibitem{sutton2018reinforcement}
R.~S. Sutton and A.~G. Barto, \emph{Reinforcement learning: An introduction}, 2nd~ed.\hskip 1em plus 0.5em minus 0.4em\relax \!\!\!Cambridge {MA}, USA: MIT press, 2018.

\bibitem{kim2022comparison}
J.~Kim, D.~Kim, Y.~Song, H.~Shim, H.~Sandberg, and K.~H. Johansson, ``Comparison of encrypted control approaches and tutorial on dynamic systems using {L}earning {W}ith {E}rrors-based homomorphic encryption,'' \emph{Annu. Rev. Control}, vol.~54, pp. 200--218, 2022.

\bibitem{schluter2023brief}
N.~Schl{\"u}ter, P.~Binfet, and M.~S. Darup, ``A brief survey on encrypted control: From the first to the second generation and beyond,'' \emph{Annu. Rev. Control}, vol.~56, 2023.

\bibitem{Kogiso2015CybersecurityEO}
K.~Kogiso and T.~Fujita, ``Cyber-security enhancement of networked control systems using homomorphic encryption,'' \emph{Proc. 54th IEEE Conf. Decision Control}, pp. 6836--6843, 2015.

\bibitem{Teranishi2021InputOutputHF}
K.~Teranishi, T.~Sadamoto, and K.~Kogiso, ``Input--output history feedback controller for encrypted control with leveled fully homomorphic encryption,'' \emph{IEEE Trans. Control Netw. Syst.}, vol.~11, no.~1, pp. 271--283, 2024.

\bibitem{schluter2021encrypted}
N.~Schl{\"u}ter, M.~Neuhaus, and M.~S. Darup, ``Encrypted dynamic control with unlimited operating time via {FIR} filters,'' in \emph{Proc. 2021 Eur. Control Conf.}, 2021, pp. 952--957.

\bibitem{kim2022dynamic}
J.~Kim, H.~Shim, and K.~Han, ``Dynamic controller that operates over homomorphically encrypted data for infinite time horizon,'' \emph{IEEE Trans. Autom. Control}, vol.~68, no.~2, pp. 660--672, 2023.

\bibitem{jang2024ring}
Y.~Jang, J.~Lee, S.~Min, H.~Kwak, J.~Kim, and Y.~Song, ``Ring-{LWE} based encrypted controller with unlimited number of recursive multiplications and effect of error growth,'' \emph{arXiv:2406.14372}, 2024.

\bibitem{ec_mpc1}
M.~{S. Darup}, A.~{Redder}, I.~{Shames}, F.~{Farokhi}, and D.~{Quevedo}, ``Towards encrypted {MPC} for linear constrained systems,'' \emph{IEEE Control Syst. Lett.}, vol.~2, no.~2, pp. 195--200, 2018.

\bibitem{ec_mpc2}
A.~B. {Alexandru}, M.~{Morari}, and G.~J. {Pappas}, ``Cloud-based {MPC} with encrypted data,'' in \emph{Proc. 57th Conf. Control Decision}, 2018.

\bibitem{ec_cooperative2}
A.~B. Alexandru, M.~S. Darup, and G.~J. Pappas, ``Encrypted cooperative control revisited,'' in \emph{Proc. 58th IEEE Conf. Decision Control}, 2019, pp. 7196--7202.

\bibitem{alexandru2020cloud}
A.~B. Alexandru, K.~Gatsis, Y.~Shoukry, S.~A. Seshia, P.~Tabuada, and G.~J. Pappas, ``Cloud-based quadratic optimization with partially homomorphic encryption,'' \emph{IEEE Trans. Autom. Control}, vol.~66, 2021.

\bibitem{schluter2022pidtuning}
N.~Schl{\"u}ter, M.~Neuhaus, and M.~S. Darup, ``Encrypted extremum seeking for privacy-preserving pid tuning as-a-service,'' in \emph{Proc. 2022 Eur. Control Conf.}, 2022, pp. 1288--1293.

\bibitem{suh2021sarsa}
J.~Suh and T.~Tanaka, ``{SARSA(0)} {Reinforcement Learning} over {Fully Homomorphic Encryption},'' in \emph{2021 SICE Int. Symp. Control Syst.}, 2021, pp. 1--7.

\bibitem{suh2021encrypted}
------, ``Encrypted {Value Iteration} and {Temporal Difference Learning} over {Leveled Homomorphic Encryption},'' in \emph{Proc. 2021 Amer. Control Conf.}, 2021, pp. 2555--2561.

\bibitem{dzurkova2024approximated}
D.~Dzurkov{\'a}, P.~Val{\'a}bek, O.~M{\'e}sz{\'a}ros, M.~Kal{\'u}z, and M.~Klau{\v{c}}o, ``Approximated explicit {NMPC} via {Reinforcement Learning} for homomorphically encrypted process control,'' in \emph{Proc. 63rd IEEE Conf. Decision Control}, 2024, pp. 4574--4581.

\bibitem{gentry}
C.~Gentry, ``{A Fully Homomorphic Encryption Scheme},'' Ph.D. dissertation, Dept. Comput. Sci., Stanford Univ., Stanford, CA, USA, 2009.

\bibitem{cheon2017homomorphic}
J.~H. Cheon, A.~Kim, M.~Kim, and Y.~Song, ``{Homomorphic Encryption for Arithmetic of Approximate Numbers},'' in \emph{Int. Conf. Theory Appl. Cryptol. Inf. Secur.}, 2017, pp. 409--437.

\bibitem{cheon2018bootstrapping}
J.~H. Cheon, K.~Han, A.~Kim, M.~Kim, and Y.~Song, ``Bootstrapping for {Approximate Homomorphic Encryption},'' in \emph{Annu. Int. Conf. Theory Appl. Cryptograph. Techn.}, 2018, pp. 360--384.

\bibitem{LyubPeik10}
V.~Lyubashevsky, C.~Peikert, and O.~Regev, ``On ideal lattices and learning with errors over rings,'' in \emph{J. ACM}, vol.~60, no.~6, 2010.

\bibitem{puterman2014markov}
M.~L. Puterman, \emph{{Markov Decision Processes}: discrete stochastic dynamic programming}, 1st~ed.\hskip 1em plus 0.5em minus 0.4em\relax \!\!\!Hoboken, NJ, USA: Wiley, 2014.

\bibitem{hecomparison}
J.~H. Cheon, D.~Kim, D.~Kim, H.~H. Lee, and K.~Lee, ``Numerical method for comparison on homomorphically encrypted numbers,'' Cryptology {ePrint} Archive, Paper 2019/417, 2019.

\bibitem{todorov2009efficient}
E.~Todorov, ``Efficient computation of optimal actions,'' \emph{Proc. {Nat. Acad. Sci.}}, vol. 106, no.~28, 2009.

\bibitem{dupuis2011weak}
P.~Dupuis and R.~S. Ellis, \emph{A weak convergence approach to the theory of large deviations}.\hskip 1em plus 0.5em minus 0.4em\relax John Wiley \& Sons, 2011.

\bibitem{horn2012matrix}
R.~A. Horn and C.~R. Johnson, \emph{Matrix {Analysis}}, 2nd~ed.\hskip 1em plus 0.5em minus 0.4em\relax \!\!\!New York, NY, USA: Cambridge Univ. Press, 2012.

\end{thebibliography}

\end{document}